\documentclass[a4paper,UKenglish,cleveref, autoref, thm-restate]{lipics-v2021}

\hideLIPIcs  %

\graphicspath{{./fig/}}%
\usepackage{cite}
\usepackage{array}
\usepackage{dblfloatfix}
\usepackage{url}
\usepackage{booktabs}       %
\usepackage{multirow}
\usepackage{amsfonts}       %
\usepackage{nicefrac}       %
\usepackage[]{microtype}      %
\usepackage{graphicx}
\usepackage[export]{adjustbox}
\usepackage{amssymb}
\usepackage{amsmath,mathtools}
\usepackage{xspace}
\usepackage{tikz}
\usepackage{rotating}
\usepackage{soul}
\usepackage[noend]{algorithm,algpseudocode}
\usepackage{amsthm}
\newtheorem*{formalization*}{Formalization}

\usepackage{makecell}
\usepackage{arydshln}

\usepackage[thinc]{esdiff}

\hypersetup{
    colorlinks=true,
    linkcolor=blue,
    filecolor=magenta,
    urlcolor=blue,
    citecolor = blue,
}

\newcommand{\prob}[1]{\Pr[{#1}]}
\newcommand{\Ex}[1]{\mathbb{E}[{#1}]}
\newcommand{\ATE}[1]{\textsc{ATE}({#1})}
\newcommand{\CATE}[1]{\textsc{CATE}({#1})}
\newcommand{\ACATE}[1]{\textsc{ACATE}({#1})}

\newcommand{\data}{\ensuremath{\mathsf{data}}}
\newcommand{\mainalg}{\ensuremath{\mathsf{CausalSAT}}}

\newcommand{\hillclimbing}{\ensuremath{\mathsf{HillClimbing}}}
\newcommand{\causaldiag}{\ensuremath{\mathsf{causalGraph}}}
\newcommand{\estimate}{\ensuremath{\mathsf{Estimate}}}
\newcommand{\idfyandest}{\ensuremath{\mathsf{Identify\&Estimate}}}
\newcommand{\estimates}{\ensuremath{\mathsf{estimates}}}
\newcommand{\estimand}{\ensuremath{\mathsf{estimand}}}
\newcommand{\identify}{\ensuremath{\mathsf{Identify}}}
\newcommand{\refute}{\ensuremath{\mathsf{Refute}}}
\newcommand{\pass}{\ensuremath{\mathsf{passRefutation}}}
\newcommand{\emptygraph}{\ensuremath{\mathsf{EmptyGraph}}}
\newcommand{\addedge}{\ensuremath{\mathsf{AddEdge}}}
\newcommand{\whitelist}{\ensuremath{\mathsf{whiteList}}}
\newcommand{\blacklist}{\ensuremath{\mathsf{blackList}}}
\newcommand{\bestdiag}{\ensuremath{\mathsf{bestGraph}}}
\newcommand{\bestscore}{\ensuremath{\mathsf{bestScore}}}
\newcommand{\computescore}{\ensuremath{\mathsf{ComputeScore}}}
\newcommand{\fail}{\ensuremath{\mathsf{Failure}}}
\newcommand{\score}{\ensuremath{\mathsf{score}}}
\newcommand{\updated}{\ensuremath{\mathsf{updated}}}
\newcommand{\false}{\ensuremath{\mathsf{False}}}
\newcommand{\true}{\ensuremath{\mathsf{True}}}
\newcommand{\perturbdiag}{\ensuremath{\mathsf{perturbGraph}}}
\newcommand{\currentdiag}{\ensuremath{\mathsf{currentGraph}}}
\newcommand{\pseudoand}{\textbf{and }}
\newcommand{\pseudoor}{\textbf{or }}
\newcommand{\treat}{\ensuremath{\mathsf{treatment}}}

\newcommand{\outcome}{\ensuremath{\mathsf{outcome}}}
\newcommand{\query}{\ensuremath{\mathsf{query}}}
\newcommand{\bdoorvar}{\ensuremath{\mathsf{backdoorVar}}}
\newcommand{\bdoorset}{\ensuremath{\mathsf{backdoorSet}}}
\newcommand{\addtolist}{\ensuremath{\mathsf{AddToList}}}
\newcommand{\bdoorsize}{\ensuremath{\mathsf{backdoorSize}}}
\newcommand{\candset}{\ensuremath{\mathsf{candidateSet}}}
\newcommand{\eligibleset}{\ensuremath{\mathsf{eligibleSet}}}
\newcommand{\pseudobreak}{\textbf{break}}
\newcommand{\sizeof}{\ensuremath{\mathsf{SizeOf}}}
\newcommand{\linreg}{\ensuremath{\mathsf{LinearRegression}}}
\newcommand{\coefs}{\ensuremath{\mathsf{coefficients}}}
\newcommand{\getcoef}{\ensuremath{\mathsf{GetCoefficientOf}}}
\newcommand{\bdoorgraph}{\ensuremath{\mathsf{backdoorGraph}}}

\newcommand{\Exp}[1]{\ensuremath{\mathbb{E}\left[{#1}\right]}}

\newcommand{\branchheu}{\ensuremath{\mathsf{Branching}}\xspace}
\newcommand{\restartheu}{\ensuremath{\mathsf{Restart}}\xspace}
\newcommand{\activity}{\ensuremath{\mathsf{Activity}}\xspace}

\newcommand{\uipuse}{\ensuremath{\mathsf{UIP}}\xspace}
\newcommand{\prop}{\ensuremath{\mathsf{Propagation}}\xspace}
\newcommand{\lastouch}{\ensuremath{\mathsf{LastTouch}}\xspace}
\newcommand{\timeinsolver}{\ensuremath{\mathsf{Time}}\xspace}
\newcommand{\utility}{\ensuremath{\mathsf{Utility}}\xspace}
\newcommand{\glue}{\ensuremath{\mathsf{LBD}}\xspace}
\newcommand{\size}{\ensuremath{\mathsf{Size}}\xspace}

\newcommand{\name}{\ensuremath{\mathsf{CausalSAT}}}

\title{Explaining SAT Solving Using Causal Reasoning}

\author{Jiong Yang}{National University of Singapore}{jiong@comp.nus.edu.sg}{}{}

\author{Arijit Shaw}{Chennai Mathematical Institute, India
\and IAI, TCG-CREST, Kolkata, India}{if.arijit@gmail.com}{}{}

\author{Teodora Baluta}{National University of Singapore}{teobaluta@comp.nus.edu.sg}{}{}

\author{Mate Soos}{National University of Singapore}{}{}{}

\author{Kuldeep S. Meel}%
{National University of Singapore}{}{}{}

\authorrunning{J. Yang et al.} %

\Copyright{Jiong Yang, Arijit Shaw, Teodora Baluta, Mate Soos, and Kuldeep S. Meel} %



\category{} %

\relatedversion{This manuscript is the full version of the paper accepted at SAT 2023.} %

\supplementdetails[subcategory={Source Code}]{Software}{https://github.com/meelgroup/causalsat}

\funding{This work was supported in part by National Research Foundation Singapore under its NRF Fellowship Programme [NRF-NRFFAI1-2019-0004], Ministry of Education Singapore Tier 2 Grant [MOE-T2EP20121-0011], and Ministry of Education Singapore Tier 1 Grant [R-252-000-B59-114]. Teodora Baluta is also supported by the Google PhD Fellowship.
}%

\acknowledgements{We are thankful to Tim van Bremen for providing detailed feedback on the early drafts of the paper and grateful to the anonymous reviewers for their constructive comments to improve this paper. We extend sincere gratitude to Jakob Nordström and other participants of the program \textit{Satisfiability: Theory, Practice, and Beyond} at the Simons Institute for the Theory of Computing for insightful discussions. Part of the work was done during Arijit Shaw's internship at the National University of Singapore. The computational work for this article was performed on resources of the National Supercomputing Centre, Singapore \url{https://www.nscc.sg}.}%

\nolinenumbers %

\EventEditors{Meena Mahajan and Friedrich Slivovsky}
\EventNoEds{2}
\EventLongTitle{26th International Conference on Theory and Applications of Satisfiability Testing (SAT 2023)}
\EventShortTitle{SAT 2023}
\EventAcronym{SAT}
\EventYear{2023}
\EventDate{July 4--8, 2023}
\EventLocation{Alghero, Italy}
\EventLogo{}
\SeriesVolume{271}
\ArticleNo{28}

\newcommand{\cb}[0]{\textsf{CrystalBall}\xspace}

\begin{document}

\maketitle

\begin{abstract}

The past three decades have witnessed notable success in designing efficient SAT solvers, with modern solvers capable of solving industrial benchmarks containing millions of variables in just a few seconds. 
The success of modern SAT solvers owes to the widely-used CDCL algorithm, which lacks comprehensive theoretical investigation. Furthermore, it has been observed that CDCL solvers still struggle to deal with specific classes of benchmarks comprising only hundreds of variables, which contrasts with their widespread use in real-world applications. Consequently, there is an urgent need to uncover the inner workings of these seemingly {\em weak} yet {\em powerful} black boxes.

In this paper, we present a first step towards this goal by introducing an approach called {\name}, which employs causal reasoning to gain insights into the functioning of modern SAT solvers. {\name} initially generates observational data from the execution of SAT solvers and learns a structured graph representing the causal relationships between the components of a SAT solver. Subsequently, given a query such as whether a clause with low literals blocks distance (LBD) has a higher clause utility, {\name} calculates the causal effect of LBD on clause utility and provides an answer to the question. We use {\name} to quantitatively verify hypotheses previously regarded as ``rules of thumb'' or empirical findings, such as the query above or the notion that clauses with high LBD experience a rapid drop in utility over time. Moreover, {\name} can address previously unexplored questions, like which branching heuristic leads to greater clause utility in order to study the relationship between branching and clause management.
Experimental evaluations using practical benchmarks demonstrate that {\name} effectively fits the data, verifies four ``rules of thumb'', and provides answers to three questions closely related to implementing modern  solvers.
\end{abstract}
\section{Introduction}

Boolean Satisfiability (SAT) is a fundamental problem in computer science that involves determining whether there exists an assignment $\sigma$ that satisfies a given Boolean formula $F$. The applications of SAT are vast and varied, including but not limited to bioinformatics~\cite{lynce2006sat}, AI planning~\cite{kautz1992planning}, and hardware and system verification~\cite{biere1999symbolic,clarke2004tool}.
The seminal work of Cook~\cite{cook1971complexity} demonstrated that SAT is NP-complete. Unsurprisingly,  early algorithmic methods such as local search and the DPLL paradigm~\cite{dpll} faced significant scalability challenges in practice. In the early '90s, the introduction of Conflict Driven Clause Learning (CDCL)~\cite{silva} ushered in a new era of interest from both theoreticians and practitioners. The outcome of this development is the emergence of effective heuristics that empower SAT solvers to handle challenging instances across various domains. This remarkable advancement is widely recognized as the {\em SAT revolution}~\cite{glucose,lingeling,minisat,maple,grasp,chaff, chronobt ,silva }.

The modern CDCL SAT solvers owe their performance to well-designed and tightly integrated core components: \emph{branching}\cite{maple,grasp}, \emph{phase selection}\cite{pipatsrisawat2007lightweight}, \emph{clause learning}\cite{glucose09,luo2017effective}, \emph{restarts}\cite{glucose,gomes1998boosting,huang2007effect,maple1}, and \emph{learnt clause cleaning}~\cite{glucose18,ohthesis}. Consequently, the progress  has been driven largely by the continuous improvement of heuristics for these core components.
The annual SAT competition~\cite{satcompetition} has provided evidence of a pattern in which the development of heuristics for one core component necessitates and encourages the design of new heuristics for other components to ensure tight integration. Such progress, however, necessitated the solvers to be complex, whereas a  typical modern SAT solver is made up of approximately $15 - 30$ thousand lines of code and employs a wide variety of strategies to solve different types of problems. 
Therefore, to make further progress in the design of SAT solvers, it is crucial to understand the inner workings of these SAT solvers. 

Traditional complexity-theoretic studies focus on analyzing the limitations of SAT solvers: such as  determining the explanations for why SAT solvers struggle for certain instances~\cite{EGG+18}. These approaches, however, fail to inform reasons for their success in solving industrial-size instances.  Several prior studies provide intuitive hypotheses on why certain SAT-solving heuristics work well on some benchmarks~\cite{satzilla,crystalball,ohthesis}.
These hypotheses are mostly derived from the researcher or developer's intuition and empirically tested in an ad-hoc manner.
The tight coupling of heuristics has been a crucial part of the development of modern solvers~\cite{cms,cms20}.
For instance, Biere et al.~\cite{biere2015evaluating} found that certain branching heuristics work particularly well in combination with specific restart heuristics. Similarly, Oh~\cite{oh2015between} observed that in order to achieve the best solver performance, the decay factor used in activity calculations must be adjusted according to the active restart heuristics.

Motivated by the role played by empirical studies in the development of new heuristics, we take the approach of usage of data-driven methods to develop understandings of the behavior of SAT solvers. Our work relies on the {\cb} framework~\cite{crystalball} that provides white-box access to the execution of a state-of-the-art SAT solver, thereby enabling large-scale data collection.
Given the recent development in the field of causal reasoning~\cite{P09}, we investigate \textit{whether it is possible to develop a framework that may use white-box access to the execution of SAT solving to generate causal reasoning to explain the interplay among the features and heuristics of SAT solving?}

Several studies have investigated the behavior of modern solvers. Elffers et al.~\cite{EGG+18} explored the effectiveness of various solver components by employing a large set of theoretical benchmarks, with the running time serving as the parameter to evaluate their utility. Simon~\cite{S14} focused on analyzing the generated clauses' utility by examining the proofs produced by a SAT solver. Kokkala et al.~\cite{kokkala2020using} utilized DRAT-trim-based proofs to investigate the impact of restart and learned clause cleaning heuristics on the length of trimmed proofs. They also investigated the relationship between different parameters and the probability of a clause being included in the trimmed proof.
Soos et al.~\cite{crystalball} employed CrystalBall to deduce the statistical correlation between the utility of learned clauses and SAT-solving features.
Nevertheless, all of these studies were based on correlation analysis and lacked a causal reasoning perspective in understanding solver behavior.

In this work, we propose to use {\em causal reasoning} to answer what factors influence the utility of clauses in memory management in SAT solvers. Our approach combines prior knowledge and experimental data to derive a causal model, a graph whose nodes represent different components and heuristics of the solver, and its edges represent causal relations between them. The derived causal model explicitly encodes assumptions about the underlying phenomenon in clause memory management.
The causal model allows us to query it by applying a set of principled rules known as do-calculus to compute the effect of a particular factor on the utility of clauses. These rules ensure that each query computes the right statistical quantity given the model without introducing bias in the estimates.
We find that our derived models fit the data reasonably well even when less data is available.
Using our framework, we were able to confirm several well-known hypotheses about solvers, which highlights its potential and assures its correctness. Additionally, we have used the framework to inquire about some unresolved questions related to the solving process, and it has provided us with answers to those inquiries. This highlights that our framework provides a fresh perspective on how to approach and analyze solvers.

We demonstrate the application of our approach to study four hypotheses that stem from observations or assumptions in prior work and encode these as causal queries.
We find that clauses with lower Literal Block Distance (LBD) have greater utility; thus, lower LBD clauses should be prioritized in clause memory management. We find that small clauses have greater utility, regardless of their LBD. We also justify the hierarchical clause memory management that high LBD clauses should be cached for a short time since their utility experiences a rapid drop over time. Moreover, we verify that LBD has a larger causal impact than the size on clause utility, and hence LBD is a preferable choice for prioritizing clauses in memory management.
We additionally study three novel questions that have not been formulated before, such as which branching heuristic results in a greater clause utility, to study the relationship between branching and clause management. We find that as a branching heuristic, Maple~\cite{maple1} leads to a greater clause utility than VSIDS~\cite{chaff}. To summarize, we demonstrate that causal reasoning allows one to evaluate the crucial choices in the design of SAT solvers.

\section{Background and Motivation} \label{sec:background}

\subsection{CDCL Solvers}

CDCL-based SAT solvers begin with an initially empty set of assignments, maintaining a partial assignment at each step. The solver incrementally assigns a subset of variables until the current partial assignment is unable to satisfy the current formula. At this point, the solver employs a backtracking mechanism to trace the cause of unsatisfiability, which is expressed as a conflict clause. Modern solvers frequently perform restarts, resetting the partial assignment to empty.

Modern SAT solvers may run for millions of conflicts during a single execution, with each conflict resulting in the learning of one or more clauses that are subsequently added to a database. As this database grows over time, it becomes necessary to periodically remove some of the less useful clauses to maintain solver efficiency. To accomplish this, various heuristics have been proposed that predict the usefulness of learned clauses. These heuristics, collectively known as \textit{learned clause cleaning} heuristics, enable the solver to identify and remove unnecessary clauses, improving performance and scalability.
A \textit{branching heuristic} is employed to determine which variable to select for each decision point in the search process. A \textit{restart heuristic} is used to determine when the solver should restart the search process from the beginning.

\subsection{Learnt Clause Cleaning}
\label{sec: motivation}

One of the most crucial parts of a CDCL-based SAT solver is the learning of clauses, which occurs rapidly. In fact, a typical SAT solver can learn millions of clauses within a minute of operation. However, the sheer volume of these clauses can hinder the solver's efficiency. To address this issue, \textit{learnt clause cleaning} heuristics were developed. These heuristics evaluate the ``utility'' of a clause based on some parameters and determine whether to retain or discard it. A heuristic may prioritize keeping clauses with higher utility while discarding those with lower utility to optimize the solver's performance.
Previous literature proposed various parameters to determine the usefulness of a clause. Een et al.~\cite{minisat} employed the notion of \textit{activity (or VSIDS)} as a surrogate for utility. Activity is a parameter that indicates a clause's involvement in recent conflicts. On the other hand, Audemard et al.~\cite{glucose} utilized \textit{literal block distance (LBD)} as a proxy for utility, where LBD represents the number of unique decision levels in the clause. Despite its success, it is difficult to tell {\em why} LBD is a good indicator of utility and disentangle its impact from other solving heuristics.
To illustrate this, we highlight some of the existing hypotheses along with novel questions about the utility of clauses below.
\begin{itemize}
	\item Clauses with small LBD have greater utility.
	\item Small clauses have greater utility.
	\item High-LBD clause experiences a rapid drop in clause utility over time.
	\item LBD has a greater impact on clause utility than clause size.
	\item What factor, other than LBD, size, and activity, has the greatest impact on clause utility?
	\item Which branching heuristic results in the greatest clause utility?
	\item Which restart heuristic results in the greatest clause utility?
\end{itemize}

It is thus natural to ask if these hypotheses and questions explain the underlying SAT-solving behavior well. Do they provide new insights into other aspects of solving? We show that in order to correctly answer such questions, we require a principled framework to even compute the right statistical quantities from data.

\subsection{Causal Model}

A causal model represents the set of features that cause an outcome and the function that measures the effect of the features on the outcome.
Prior work proposes prior domain knowledge such as causal invariant transformations~\cite{wang2022out}, score-based learning algorithms~\cite{SN11}, interventional causal representation learning~\cite{AWMB22,BDHLC22,scholkopf2021toward} or learning invariant relationships from
training datasets from different distributions~\cite{arjovsky2019invariant,peters2016causal} to construct causal models.

As with probabilistic graphical models, the underlying representation of the causal model is (1) a directed acyclic graph (DAG) $G=(V,E)$ where the vertices $V$ are the set of random variables $V=\{X_1, \ldots, X_n\}$ and the edges $(X_i, X_j) \in E$ represent causal relationships $X_i$ causes $X_j$ and (2) set of parameters $\Theta$ quantifying the relationships between the variables in $V$.
The set of variables $V$ can take either discrete or continuous values.
We say that a subset of variables $S \subset V$ cause $X_i$, if $X_i=g(\{X_j\}_{j\in S}, N_Y)$ for some function $g$ and $N_Y$ a random variable independent of all other variables. The set of functions is parameterized by $\theta\in\Theta$. 
Henceforth, we will use causal models and causal graphs interchangeably.
The absence of an edge from $X$ to $Y$ represents (conditional) independence. Due to this, causal graphs are also known as causal Bayesian networks. Causal graphs still have the probabilistic interpretation of their associative (Bayesian) counterpart: they represent a joint probability distribution of the variables in $V$, determined by their conditional independence relationships, i.e., $\mathcal{P}_\theta(V)=\Pi_i\prob{X_i|pa_{X_i}}, X_i\in V$, where $pa_{X_i}$ represents the parents of the $X_i$ vertex.

\subsection*{Causal Inference}
The causal graph has a quantitative interpretation: we use it to answer {\em do-queries} that estimate the causal effect between a {\em treatment} variable and an {\em outcome} variable. 

\noindent\textbf{{\em do}-query.} We say there is a causal effect between a treatment variable $X_i$ and an outcome variable $Y$ if under an intervention on the treatment variable $X_i$, the outcome variable changes. Given a set of variables $V=\{X_1, \ldots, X_n, Y\}$, an intervention on a variable $X_i$ is an experiment where the experimenter controls the variable $X_i \in S$ to take a value $u$ of another independent (from other variables $\in V$) variable, i.e., $X_i=u$. More generally, we can intervene on a subset of variables $S \subset V$. 
This operation has been formalized as the {\em do} operator by Pearl~\cite{P09}.
The $do$ operator thus changes the value of $X_i$ while keeping every other variable in $V$ the same, except for those directly or indirectly affected by $X_i$.
This is akin to removing the edges of the nodes in $S$ to their parents in the graph $G$, resulting in a manipulated graph.
The intervention effect on the outcome variable, e.g., $\prob{Y|do(X_i=u)}$, is then measured by data generated under the distribution of the manipulated graph. 
$Do$-queries are thus fundamentally different from observational or conditional queries in that it requires being able to set the value of the variable $X_i$ to a potentially unobserved value and keeping every other unaffected variable in $V$ the same.
The framework to reason about $do$-queries has been formalized in a set of rules known as the $do$-calculus which has been shown to be sound and complete~\cite{HV06}.

\noindent\textbf{Average Treatment Effect (ATE).} Given a causal graph, a treatment variable $X$, and an outcome variable $Y$, the average treatment effect measures the change of the expected value of the outcome variable when we intervene on the treatment variable and change it from a constant value $b$ to $a$. 

\begin{definition}[Average Treatment Effect]\label{def:ate}
The average treatment effect of a variable $X$ (called the treatment) on the target variable $Y$ (called the outcome) is:
\begin{align*}
 \ATE{X, Y, a, b} = \Exp{Y|do(X = a)} - \Exp{Y|do(X = b)},
\end{align*}
\end{definition}
where $a,b$ are constants for which $X$ is defined. We omit the constants when the query is over the  domain of $X$.
Following~\Cref{def:ate}, the conditional average treatment effect (CATE) measures the average outcome conditioned on a variable $W$ of interest.
\begin{definition}[Conditional Average Treatment Effect]\label{def:cate}
	The conditional average treatment effect of $X$ on $Y$ conditioned on a variable of interest $W$ is:
	\begin{align*}
		\CATE{X, Y, W, a, b} = \Exp{Y|do(X = a), W} - \Exp{Y|do(X = b), W}.
	\end{align*}
\end{definition}

\subsection*{Why Causal Analysis?}

Deriving the right quantities to compute is challenging without a proper framework. As an example, we first consider {\glue}, {\prop}, i.e., the number of times a clause is involved in propagation, and {\lastouch}, i.e., the number of conflicts since a clause is involved in conflict analysis.
We collect data corresponding to these variables during the SAT-solving process of 80 instances by varying several of the existing heuristics.
Suppose we want to compute the impact of {\prop} on {\lastouch}, that is, the expected change of {\lastouch} if {\prop} increases by one.
This quantity can be written as the expected value of {\lastouch} when {\prop} increased from one to two: $\Exp{\lastouch \mid \prop=2}-\Exp{\lastouch \mid \prop=1}=-431.29$.

However, there is a common cause, {\glue}, that affects both {\prop} and {\lastouch}.
Having this knowledge at hand, we need to ``control'' for {\glue}. ``Controlling'' for a common cause in statistical measurements means binning over the observed values of the common cause. Thus, the right quantity to compute is:
$\sum_z ( \Exp{\lastouch \mid \prop=2, \glue=z}-\Exp{\lastouch \mid \prop=1, \glue=z} ) \Pr[\glue=z]$.
Notice that we can represent such relationships between variables using a {\em causal model}, i.e., a graph where the nodes are variables and the edges represent causal relations. We show the graph corresponding to our example in Figure~\ref{fig:example-1}.
Estimating this statistical quantity using a linear regression model that fits the data yields $-2.92$. Thus, our first attempt at estimating the effect did not model the underlying process precisely.

Let us now consider a fourth variable, {\activity} that is also correlated with {\prop} and {\lastouch} (see Figure~\ref{fig:example-2}). Should we control for this variable as well?
Suppose we controlled for it and estimated again the effect of {\prop} on {\lastouch} as the following:
\begin{align*}
    \sum_{z,a} (&\Exp{\lastouch \mid \prop=2, \glue=z, \activity=a}-\\
    &\Exp{\lastouch \mid \prop=1, \glue=z, \activity=a}) \Pr[\glue=z] \Pr[\activity=a]
\end{align*}

On the collected data, this results in an estimated effect of $-1.96$. However, {\activity} should not be controlled for since it is not a common cause, but rather it is affected by both {\prop} and {\lastouch}. Intuitively, during SAT solving, the {\activity} is computed {\em after} these two variables. The correct estimate is $\sum_z ( \Exp{\lastouch \mid \prop=2, \glue=z}-\Exp{\lastouch \mid \prop=1, \glue=z} ) \Pr[\glue=z]$ = -2.92. Over-controlling biased the estimated effect. It is thus crucial to identify the correct set of variables to control.

\begin{figure}[h!]
\captionsetup[subfigure]{justification=centering,skip=-5pt}
\centering
\begin{subfigure}{.4\textwidth}
  \centering
  \vspace{.2cm}
  \includegraphics[scale=0.3]{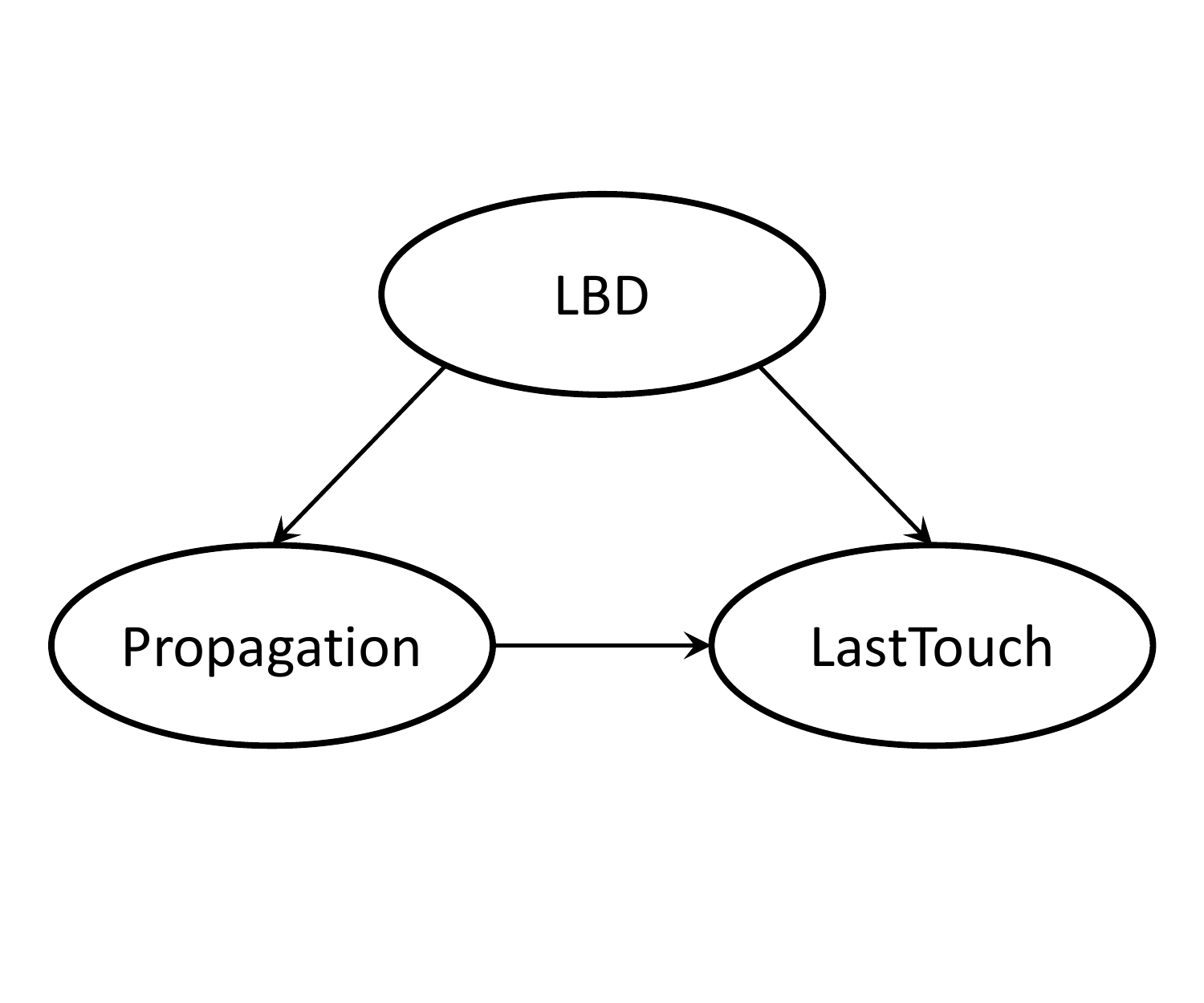}
  {\fontsize{6.5pt}{7.8pt} \medmuskip=0mu \thinmuskip=0mu \thickmuskip=0mu
    $\sum_z ( \Ex{\mathsf{LastTouch} \mid \mathsf{Propagation}=2, \mathsf{LBD}=z}$ 
    $ \phantom{abcde} -\  \mathbb{E}[\mathsf{LastTouch} \mid \mathsf{Propagation}=1, \mathsf{LBD}=z] ) \Pr[z]$\\ \ }
  \caption{}
  \label{fig:example-1}
\end{subfigure}%
\begin{subfigure}{.54\textwidth}
  \centering
  \vspace{-1cm}
  \includegraphics[scale=0.3]{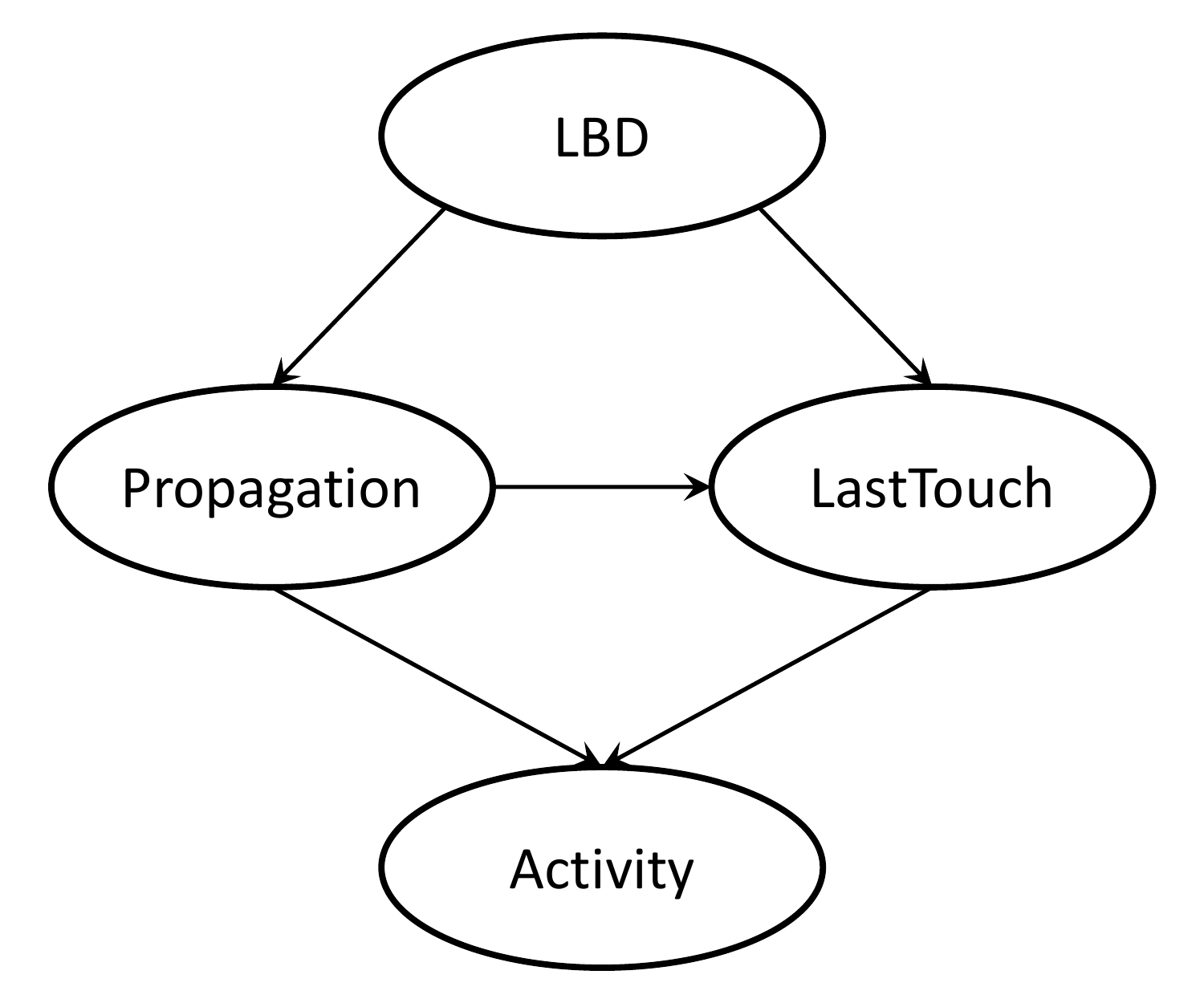}
  {\fontsize{6pt}{7.2pt} \medmuskip=0mu \thinmuskip=0mu \thickmuskip=0mu 
    $\sum_{z,a} ( \Ex{\mathsf{LastTouch} \mid \mathsf{Propagation}=2, \mathsf{LBD}=z, \mathsf{Activity}=a}$ 
    $ \phantom{abcdefghi} -\  \Ex{\mathsf{LastTouch} \mid \mathsf{Propagation}=1, \mathsf{LBD}=z, \mathsf{Activity}=a} ) \Pr[z] \Pr[a]$\\ \ }
  \caption{}
  \label{fig:example-2}
\end{subfigure}
\caption{Causal models help correctly identify the set of variables to control for. (a) Estimated effect of {\prop} on {\lastouch} when {\glue} is a cause of both. (b) Over-controlling while estimating the effect of {\prop} on {\lastouch} when {\activity} is introduced. We use the notation $\Pr[z]$ and $\Pr[a]$ as a shorthand for $\Pr[\glue=z]$ and $\Pr[\activity=a]$, respectively.}\label{fig:example} 
\end{figure}

\section{Causal Reasoning for SAT Solving} \label{sec:our-approach}
LBD, size, and activity have long been employed to predict the utility of learned clauses in SAT solvers. 
For instance, Kissat~\cite{kissat} consistently retains clauses with $\text{LBD} \le 2$, 
while CryptoMiniSat~\cite{cms} always preserves clauses with $\text{LBD} \le 3$. 
These heuristics constitute a critical component and substantially enhance the performance of modern SAT solvers. 
Despite being considered a ``rule of thumb'' for solver design, 
whether a low-LBD clause yields higher utility than a high-LBD clause still remains uncertain.

To better comprehend the functionality of these factors and promote the discovery of new factors, 
we propose a causality-based approach to examine the influence of these factors on the utility of learned clauses. 
Our objective is to use this method to address queries closely tied to the implementation of modern SAT solvers, 
such as whether a high or low LBD clause offers greater utility and what other factors significantly impact clause utility. 
The findings can then be leveraged to inform the design of SAT solvers.

\paragraph*{Overview}
We introduce an approach designed to address SAT-related inquiries. 
Figure~\ref{fig:system-architecture} displays our approach's overview, 
while Algorithm~\ref{alg: main} depicts the pseudo-code for our prototype. 
Initially, we generate observational data from a SAT solver at Line~\ref{ln: main data generation}. 
This data records factor values and clause utility (formally defined in Section~\ref{sec: data gen}).
Subsequently, we construct a causal graph representing the causal relationship between all variables, including factors and clause utility, 
Line~\ref{ln: main learn structure}, based on our prior knowledge and observational data. 
Meanwhile, we formulate each SAT-related question into a causal query at Line~\ref{ln: main formulate}. 
For instance, the LBD-related question can be formulated as a computation of the causal effect from LBD to clause utility.
From Lines~\ref{ln: main identify} to \ref{ln: main return failure}, we calculate a causal effect estimate. 
First, we identify a mathematical expression for the estimation known as an \emph{estimand}. 
Next, the estimand is evaluated on observational data to calculate the causal effect estimate. 
If the resulting estimate passes the refutation test, 
it is returned as the final causal effect; otherwise, a failure is reported.
Lastly, we return the causal effect corresponding to the original question.

\begin{figure}[h!]
	\centering
	\includegraphics[scale=0.4]{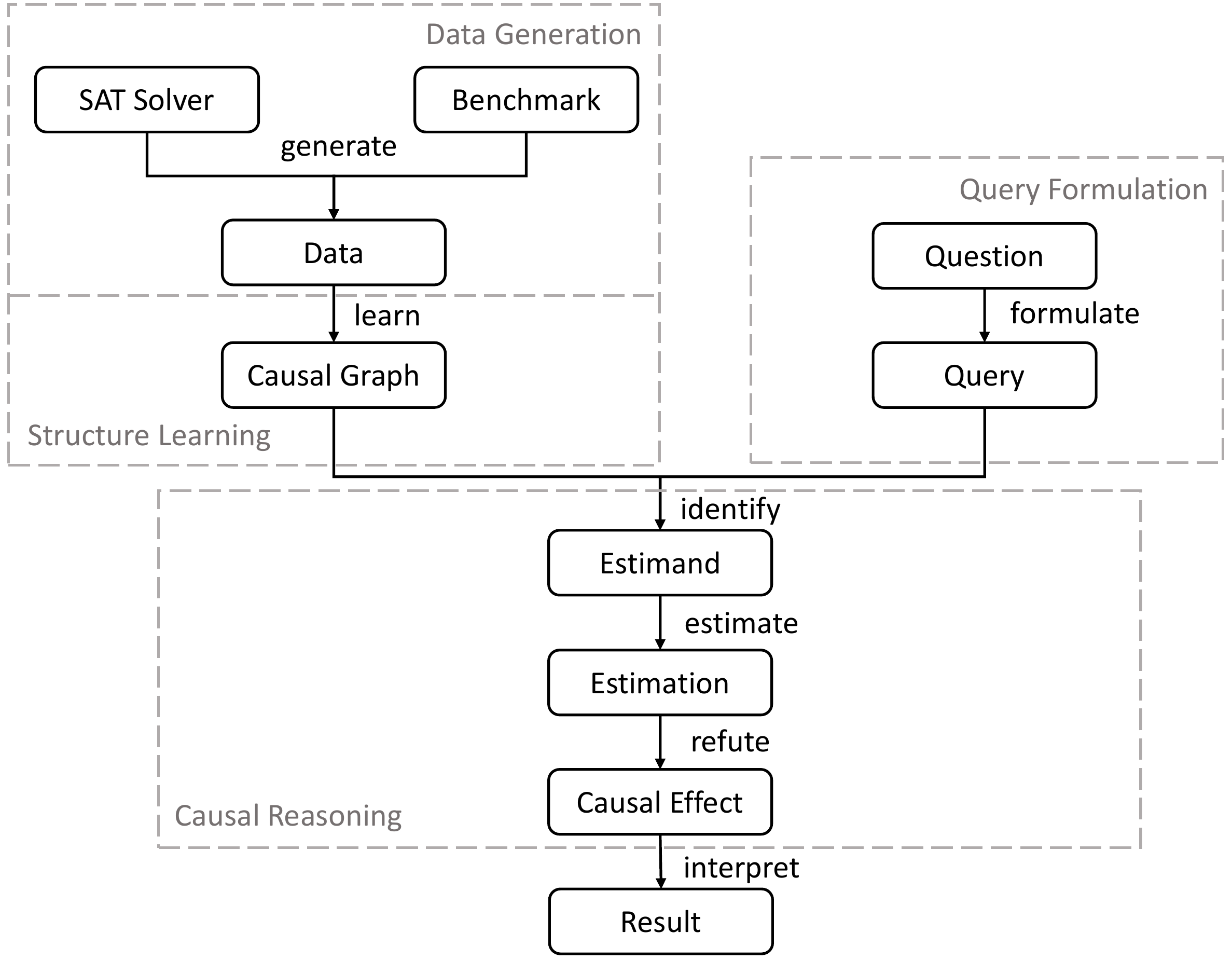}
	\caption{Our approach overview, from data generation to the causal estimate.} 
	\label{fig:system-architecture}
\end{figure}

In the subsequent subsections, we will delve into each component depicted in Figure~\ref{fig:system-architecture}. 
Section~\ref{sec: data gen} outlines the process of generating observational data from a SAT solver. 
Section~\ref{sec: structure learning} elaborates on the construction of a causal graph using our prior knowledge of SAT solvers and the generated data. 
Section~\ref{sec:query-sat} introduces the SAT-related questions of interest and describes their formulation into a causal query.
Lastly, Section~\ref{sec:causal-reasoning} demonstrates the computation of the causal effect for a given query. 

\begin{algorithm}[t]
	\caption{Our approach ($\mainalg$)}
	\label{alg: main}

	\begin{algorithmic}[1]
            \State $\data \gets$ generate observational data from a SAT solver$;$ \label{ln: main data generation}
		\State $\whitelist, \blacklist \gets \text{decode predefined and blocked edges from prior knowledge};$ \label{ln: main decode prior knowledge}
		\State $\causaldiag \gets \hillclimbing(\whitelist, \blacklist, \data);$ \label{ln: main learn structure}
            \State $\query \gets$ formulate a question$;$ \label{ln: main formulate}
            \State $\estimand \gets \identify(\query, \causaldiag)$ \label{ln: main identify}
		\State $\estimates \gets \estimate(\estimand, \causaldiag, \data);$ \label{ln: main estimate}
		\State $\pass \gets \refute(\estimates, \query, \causaldiag, \data);$ \label{ln: main refute}
		\If{$\pass$} \Return $\estimates;$ \label{ln: main return estimate}
		\Else \ \Return $\fail;$ \label{ln: main return failure}
		\EndIf
	\end{algorithmic}
\end{algorithm}

\subsection{Data Generation}
\label{sec: data gen}
To accomplish our goal of exploring causality and assess the utility of a clause, we collected data that included traces of the SAT solver's execution. These traces captured various characteristics of a clause at different stages of the instance-solving process, as well as the future utility of the clause.
We work with a set of unsatisfiable instances. For each instance, we run the SAT solver multiple times with different combinations of heuristics. To complete the data collection procedure, we employ a two-pass system. During the \emph{forward pass}, we run the SAT solver until it reaches the conclusion of unsatisfiability. 
To gather data on all the learned clauses, we deactivated the learned clause cleaning heuristic during the execution of the SAT solver. Consequently, all the learned clauses were retained throughout the solver's execution.
We store data about all the learned clauses at different snapshots of solving. However, this data lacks labeling about whether each clause was useful in proof generation. 
To address this issue, we execute a \emph{backward pass} using the proof generator DRAT-trim. In this pass, we label all the data collected in the forward pass by how many times the clause will be used in the next 10k conflicts. By combining the data from both passes, we obtain a complete dataset for our study.
As a result, each data point in our dataset includes information about the instance, the heuristics used, the learned clauses, some of its features, and the usefulness of each clause in proof generation.

\subparagraph{Features and Heuristics}
In general, the accuracy of a causal model is often improved with a larger number of features and a larger training dataset. However, many algorithms for detecting causal structures are limited in their ability to handle a large number of features. Therefore, we have selected a subset of important features that are commonly used in modern SAT solver literature. In \cref{tab:feat}, we list all the features we have in our data.
In our study, we adopted the definition of clause utility as proposed in~\cite{crystalball}, which relies on the frequency of clause usage observed in the proofs reconstructed by DRAT-trim~\cite{wetzler2014drat}. It should be noted that the proof generated by DRAT-trim may differ significantly from the proof generated by the SAT solver, necessitating caution in interpretation.

\begin{table}[!htb]
\renewcommand{\arraystretch}{1.1}
\begin{tabular}{p{0.18\linewidth}  p{0.74\linewidth}}
\toprule
{\branchheu} & Heuristics to determine the order in which variables are assigned values during the search. We consider  VSIDS~\cite{chaff} and Maple~\cite{maple1,maple2}. \\
 {\restartheu} & Heuristics to determine when the solver should restart. We consider Geometric~\cite{geomrestart} , LBD-based~\cite{glucose09}, and Luby~\cite{luby,huang2007effect} heuristics. \\ \hdashline
{\size} & number of literals in a clause. \\
{\glue} & number of distinct decision levels of literals in a clause. \\
{\activity} &  measure of clause's importance in the search process, based on the clause's involvement in recent conflicts. \\
{\uipuse} & number of times that the clause took part in a 1st-UIP conflict generation since its creation. \\
{\prop} &  number of times the clause was used in propagations.\\
{\lastouch} & number of conflicts since the clause was used during a 1st-UIP conflict clause generation. \\
{\timeinsolver} & number of conflicts since the generation of this clause.\\
{\utility} & within the next 10,000 conflicts, the number of times this clause has been used in DRAT proof generation. The number is weighted based on at which points the clauses are used. \\
\bottomrule
\end{tabular}
\caption{Data collected about clauses during solving.}
\label{tab:feat}
\end{table}

\subsection{Structure Learning}
\label{sec: structure learning}
A causal graph can be built following our prior knowledge.
For example, the clause management heuristic is a known cause for the average LBD of learned clauses.
If we preserve the clause of small LBD but remove that of large LBD during clause reduction, the average LBD will go down.
Therefore, we can add to the causal graph an edge from clause management heuristic to average LBD to represent the known causality between them.
For another instance, both branching and restart heuristics are user-defined parameters, and therefore no other variables inside the SAT solver can cause them.
We can then block incoming edges for both heuristics.
Overall, we add the following constraints based on our prior knowledge:
\begin{itemize}
	\item No incoming edges are allowed for {\branchheu} and {\restartheu} since they are given by users.
	\item No incoming edges are allowed for {\timeinsolver} because time ticks are fixed in our experiments.
	\item No incoming edges allowed for {\glue} and {\size} except for edges from {\branchheu} and {\restartheu} 
		because {\glue} and {\size} are determined before other variables.
	\item No outgoing edges are allowed for {\utility} because {\utility} follows other variables.
\end{itemize}

However, prior knowledge is insufficient to build a complete causal graph in many scenarios because the knowledge we have about the system is always incomplete, like a SAT solver.
Hence, a complete causal graph has to be learned from the observational data.
We use the hill-climbing algorithm~\cite{SD07} in our approach to learn a causal graph for its practical performance.
The algorithm starts from an initial directed graph.
In each step, we randomly add, remove, or reverse an edge of the graph to form candidate graphs.
A score function called Bayesian Information Criterion~\cite{S78}, which favors a graph with a simple structure and lower log-likelihood over the data, is then applied to these candidate graphs to select a locally optimal graph.
The process continues until a fixed point where we do not find a better graph.

Furthermore, we apply $k$-fold cross-validation to the learning process to lower the variance of the returned graph.
The dataset is initially partitioned into $k$ subsets. 
Subsequently, for each iteration, we learn a graph on $k-1$ subsets and evaluate its fitness over the remaining subset.
Finally, we obtain an averaged graph by taking a majority vote from the $k$ candidates on the existence and the direction of each edge.

There are other structure-learning algorithms available such as the PC Stable algorithm~\cite{CM14}.
However, the PC Stable algorithm always returns a graph with a few undirected edges in our experiments, which means the algorithm failed to infer the causal direction between some variables.
Due to its algorithmic nature, the hill-climbing algorithm always returns a directed graph, meeting our requirements.
A detailed comparison between structure-learning algorithms is out of the scope of this paper, 
and we refer the interested readers to~\cite{S10}.
Additionally, we present a pseudo-code of the hill-climbing algorithm in Appendix~\ref{sec:hc-alg}.

\subsection{Queries on SAT Solving} \label{sec:query-sat}
Are you curious about whether a clause with a high or low LBD possesses a greater utility? 
Alternatively, you might be interested in determining which factor—size or LBD—exerts a more significant influence on clause utility. 
Previously, one would have to rely on personal experience or experiment with heuristics to reach a conclusion. 
However, with our proposed method, you simply need to transform these inquiries into causal queries, and our causal model will provide the answers.

In this section, we initially present the different query types supported by our approach.
Following that, we explore questions closely related to SAT solver implementation 
and demonstrate the process of formulating them into causal queries.
Our primary focus lies on queries concerning clause utility, 
with the intention of examining the variables that exert a substantial influence on clause utility.

\paragraph*{Query Type}

Our causal framework accommodates three distinct query categories, as detailed below:
\begin{itemize}
    \item Average Treatment Effect (ATE) quantifies the variation in the outcome variable when the treatment variable shifts from one value to another.
    \item Conditional Average Treatment Effect (CATE) encapsulates the ATE, considering a condition for an additional variable.
    \item We propose the Averaged Conditional Average Treatment Effect (ACATE) concept to account for the need to average the CATE across all values of the conditional variable.
\end{itemize}
\begin{definition}[Averaged Conditional Average Treatment Effect]
    \label{def:acate}
    \begin{align*}
	\ACATE{X, Y, W, a, b} = \sum_w (&\Exp{Y|do(X = a), W=w} - \\ &\Exp{Y|do(X = b), W=w} ) \Pr[W=w]
    \end{align*}
\end{definition}

\subparagraph*{LBD}
In modern SAT solvers, LBD serves as a key factor for clause memory management heuristics after being introduced in Glucose~\cite{glucose09}.
It is assumed that the clause of small LBD has greater utility and is therefore favored during clause reduction.
The assumption drives us to ask the question: 
which clause, with low or high LBD, has greater utility?
The question can be translated to calculate the causal effect from LBD to clause utility, 
corresponding to an ATE query (Q1) in Table~\ref{tab:query-formulation}. 
The ATE measures the average change of the outcome (clause utility) if the treatment (LBD) changes from one to two.
Following the assumption that a clause of small LBD has greater utility, increasing LBD will decrease the utility, and hence the ATE is expected to be negative.
If the ATE is indeed negative, we prove the assumption.
Otherwise, if ATE is positive, we reach the opposite conclusion that the clause of large LBD has greater utility.
Finally, if ATE equals zero, the LBD has no effect on clause utility, and clauses with large or small LBD have the same utility.

\begin{table}[h]
    \centering
    \begin{adjustbox}{max width=\textwidth}
	\begin{tabular}{l l l}
		\toprule
		& Question & Query \\
		\midrule

            Q1 & \thead[l]{Which clause, with low or high LBD,\\ has greater utility?} 
            & $\ATE{\glue, \utility, 2, 1} < 0 $ \\ 
    
		Q2 & \thead[l]{Which type of clause, large or small,\\ has greater utility?
	       What if the LBD is fixed?} & 
            $\begin{cases}
			     \ATE{\size, \utility, 2, 1} < 0 \\
			     \ACATE{\size, \utility, \glue, 2, 1} > 0 	
		\end{cases}$ \\
  
            Q3 & \thead[l]{Which clause, with low or high LBD,\\ experiences a rapid drop in utility over time?} & 
            $\begin{cases}
		      \CATE{\timeinsolver, \utility, \glue \le 6, 10 000, 0} \ge 0 \\
			\CATE{\timeinsolver, \utility, \glue > 6, 10 000, 0} < 0
		\end{cases}$ \\

            Q4 & \thead[l]{Which factor, size or LBD,\\ has a greater impact on clause utility?}
            & $|\ATE{\size, \utility, 2, 1}| > |\ATE{\glue, \utility, 2, 1}|$ \\
  
            Q5 & \thead[l]{Which factor, besides size, LBD, and activity,\\ has the greatest impact on clause utility?} 
            & $\arg \max_{\ensuremath{\mathsf{Treatment}}\not = \utility} \{ |\ATE{\ensuremath{\mathsf{Treatment}}, \utility, 2, 1}| \}$\\

            Q6 & \thead[l]{Which branching heuristic, VSIDS or Maple,\\ results in a greater clause utility?} 
            & $\ATE{\branchheu, \utility, \text{Maple}, \text{VSIDS} }$ \\

            Q7 & \thead[l]{Which restart heuristic, Geometric, LBD-based,\\ or Luby, results in the greatest clause utility?}
            & $\begin{cases}
			\ATE{\restartheu, \utility, \text{Luby}, \text{Geometric} }\\
			\ATE{\restartheu, \utility, \text{Luby}, \text{LBD-based} }\\
                \ATE{\restartheu, \utility, \text{Geometric}, \text{LBD-based} }
		\end{cases}$ \\
		\bottomrule	
	\end{tabular}
    \end{adjustbox}
    \caption{Causal queries for questions on SAT solving.}
    \label{tab:query-formulation}
\end{table}

\subparagraph*{Size}
Modern solvers assume a small clause has higher utility than a large clause,
which brings us to the question:
which type of clause, large or small, has a greater utility?
The question is formulated into the first ATE query (Q2) in Table~\ref{tab:query-formulation}, 
which implies a decreasing utility when the clause size increases from one to two.
If the resulting ATE is negative, a small clause is proved to have greater utility.
Otherwise, a large clause has a greater utility if the ATE is positive.
Lastly, the size does not affect utility if the ATE is zero.

A historical implementation of Glucose~\cite{AS09} prioritizing large clauses of the same LBD outperformed the one favoring small clauses, 
which implies that large clauses have greater utility given a fixed LBD. 
The observation is counter-intuitive because small clauses are usually assumed to be more useful, 
and people think it still holds when LBD is fixed.
Consequently, we ask the question:
which type of clause, large or small, has greater utility given a fixed LBD?
The question is formulated into an ACATE query (Q2) in Table~\ref{tab:query-formulation},
which implies a decreasing utility when clause size increases from one to two conditioned on a fixed LBD.
The hypothesis will be verified if the ACATE is indeed positive or disproved otherwise.

\subparagraph*{LBD over Time}
Modern SAT solvers adopt multi-tier clause memory management heuristic: 
more useful clauses are preserved for a long time, 
while less useful clauses are cached for a short time. 
Clauses are classified based on LBD.
For example, Kissat~\cite{kissat} stores learned clauses with $\glue \le 6$ in the long-time tier with the rest in the short-time tier.
This observation implies that low-LBD clauses are expected to maintain their utility for a long time, 
while conversely, the utility of large-LBD clauses is expected to drop quickly over time.

The implication drives us to ask the question:
Which clause, with low or high LBD, experiences a rapid drop in utility over time?
Following that, we formulate the question into (Q3) in Table~\ref{tab:query-formulation}.
The first query formulates that small-LBD clauses do not experience a decrease in utility when time ticks from 0 to 10,000,
while the second one encodes that large-LBD clauses experience a utility drop over time.
We model the time period as 10,000 conflicts because the data is collected per 10,000 conflicts.
If the answers to two queries are both yes, the hypothesis is verified from a causal perspective.

\subparagraph*{Size vs. LBD}
In modern SAT solvers, variants of clause management heuristics are based on size and LBD,
but a question remains open: which has a greater impact on clause utility?
We formulate the question into (Q4) in Table~\ref{tab:query-formulation} to compare the causal effects on clause utility between size and LBD.
If the size (resp. LBD) has a greater effect on clause utility, we suggest clause size (resp. LBD) be involved more in clause management.
Note that we compare the effects over normalized data to eliminate the unfairness due to different orders of magnitude for size and LBD.

\subparagraph*{Beyond Size, LBD, and Activity}
Modern SAT solvers consider size, LBD, and activity as the key factors in predicting clause utility,
while other factors received less attention.
To motivate new clause management heuristics, we utilize {\name} to identify new factors for predicting clause utility.
As shown in (Q5) of Table~\ref{tab:query-formulation}, we formulate an optimization query to find the treatment 
that has the largest effect on clause utility.
The discovery may inspire the design of new clause management heuristics.

\subparagraph*{Branching Heuristics}
Modern SAT solvers can only deploy one branching heuristic at a time,
which defers the hard choice among heuristics to users. 
In this work, we provide suggestions from a causal perspective.
We ask the question: 
which branching heuristic results in a greater clause utility?
We consider two prevalent heuristics, VSIDS~\cite{chaff} and Maple~\cite{maple1,maple2}.
Following that, we formulated the question into an ATE query (Q6) in Table~\ref{tab:query-formulation}.
The query evaluates the effect on clause utility when the branching heuristic changes from VSIDS to Maple.
A positive ATE indicates Maple yields a higher utility,
while a negative value implies VSIDS demonstrates a superior utility. 
Otherwise, the ATE equals zero, which implies the branching heuristic has no impact on clause utility.

\subparagraph*{Restart Heuristics}
Considering three dominant restart heuristics Geometric~\cite{geomrestart}, LBD-based~\cite{glucose09}, and Luby~\cite{luby,huang2007effect},
we provide suggestions on the choice of restart heuristics
by asking the question:
Which restart heuristic, Geometric, LBD-based, or Luby, results in the greatest clause utility?
Following that, we formulate the question into three ATE queries (Q7) in Table~\ref{tab:query-formulation}, 
which compares the causal effects on clause utility among three heuristics.
For example, if the first two ATEs are positive, Luby leads to higher utility than Geometric and LBD-based.
Hence Luby is a preferred heuristic for high clause utility.
Similar arguments apply to Geometric and LBD-based heuristics.

\subsection{Causal Reasoning} \label{sec:causal-reasoning}

In this section, we outline the process for calculating a query with the aid of a causal graph and observational data. 
Initially, we demonstrate the identification of controlled variables that bias the effect estimation. 
Following that, we exemplify the calculation of a causal effect using linear regression and the identified variables. 
Lastly, we discuss the utilization of refutation tests for validating the employed estimation technique.

\subsubsection{Identification}

A causal query asks to calculate the effect of treatment on the outcome (with some conditions).
The naive way to do the calculation is to intervene in the treatment, i.e., change the treatment from one value to the other
and observe the variation in outcome.
However, intervening in treatment is infeasible in many applications like SAT solving,
where it's impossible to increase the size of a clause and then observe the changes in utility.
Alternatively, we can estimate the effect from the observational data.
As motivated in Section~\ref{sec: motivation}, we have to identify a set of \emph{controlled} variables to eliminate bias.

There are three methods to identify the controlled variables: backdoor, frontdoor, and instrumental variable identification.
We use backdoor identification in our approach because frontdoor and instrumental variable identification always failed to find these variables in our experiments.
Interested readers may refer to~\cite{P09b} for the details of other methods.
The backdoor algorithm returns a set of variables called \emph{backdoor set}, 
which is the smallest set separating the treatment and effect after removing the paths from the treatment to the outcome.
These paths carry the pure effect from treatment to the outcome.
A backdoor set represents a set of variables that exert an implicit effect on the treatment and outcome, such as the common cause {\glue} in Figure~\ref{fig:example-1}.
Controlling these variables ensures an unbiased estimation of the causal effect. 
Using a backdoor set, we can convert an ATE query into an expression called \emph{estimand}, 
which can be evaluated over observational data to estimate the causal effect.
Details of the backdoor identification algorithm are deferred to Appendix~\ref{sec:est-alg}.

\subsubsection{Effect Estimation}
\label{sec: estimation}
We utilize linear regression as the estimator, given its effectiveness and interpretability.
As a widely used estimator for causal effects, linear regression delivers remarkably strong performance. 
Its estimate is simple to interpret, signifying the average shift in the outcome when the treatment rises by one unit. 
Consequently, we demonstrate that the coefficient $\beta$ corresponding to the treatment variable $X$ signifies the ACATE value.
\begin{lemma}
	\label{lm:cate}
	Suppose $Z \cup W$ is a backdoor set and $\Exp{Y \mid X = x, Z = z, W = w}$ is linear, 
        i.e., $\Exp{Y \mid X = x, Z = z, W = w} = \beta_0 + \beta_1 x + \beta^T_2 z + \beta^T_3 w$. 
	We obtain 
	\begin{align*}
		\ACATE{X, Y, W, a, b} = (a-b)\beta_1
	\end{align*}
\end{lemma}
We defer the proof to Appendix~\ref{sec:appendix-proof} due to the page limit.

To identify a backdoor set that incorporates the conditional variable $W$, 
we can designate $W$ as an element of the set and seek a valid backdoor set. 
In practice, we presume $\Exp{Y \mid X = x, Z = z, W = w}$ to be linear and employ $(a-b)\beta_1$ as an estimate of ACATE. 
This conclusion is also applicable to ATE and CATE, with the proof detailed in Appendix.
To enhance confidence in statistical estimations, we conduct refutation tests to verify our estimation method. 
For example, we can assess whether the estimate changes substantially upon adding an independent random variable to the data as a common cause of the treatment and outcome. Interested readers may refer to Appendix~\ref{sec:appendix-refute} for more details.

\section{Experimental Evaluation} \label{sec:experiment}

We implement a prototype called {\name} to evaluate our approach over practical instances. {\name} is built on top of a structure-learning package, \ensuremath{\mathsf{bnlearn}} ~\cite{S10}, and a causal-reasoning package, \ensuremath{\mathsf{DoWhy}}~\cite{AE19}.
We first present {\name} fits the dataset and then use it to answer the queries in Table~\ref{tab:query-formulation}.

For our experiments, we used a high-performance computing cluster. We collected traces of the run using instances from the SAT competition benchmarks. Because of the construction of our framework, we could only use unsatisfiable instances. Additionally, the process of collecting traces added significant overhead to the SAT-solving process, preventing us from gathering data from instances that took more than an hour to solve in a modern SAT solver or instances that could be solved within a few seconds. To ensure a variety of data, we chose to use $80$ UNSAT instances from the SAT competition benchmarks from 2014-'18. We allowed a 2-hour timeout for all instances to run and record the traces. Afterward, we sampled data from those traces to ensure an equal representation of all types of instances. In total, we gathered $4$ million data points.
Lastly, we ran the hill-climbing algorithm using $10$-fold cross-validation and obtained the causal graph in Figure~\ref{fig: structure hc}.

\begin{figure}[h!]
	\centering
	\includegraphics[scale=0.35]{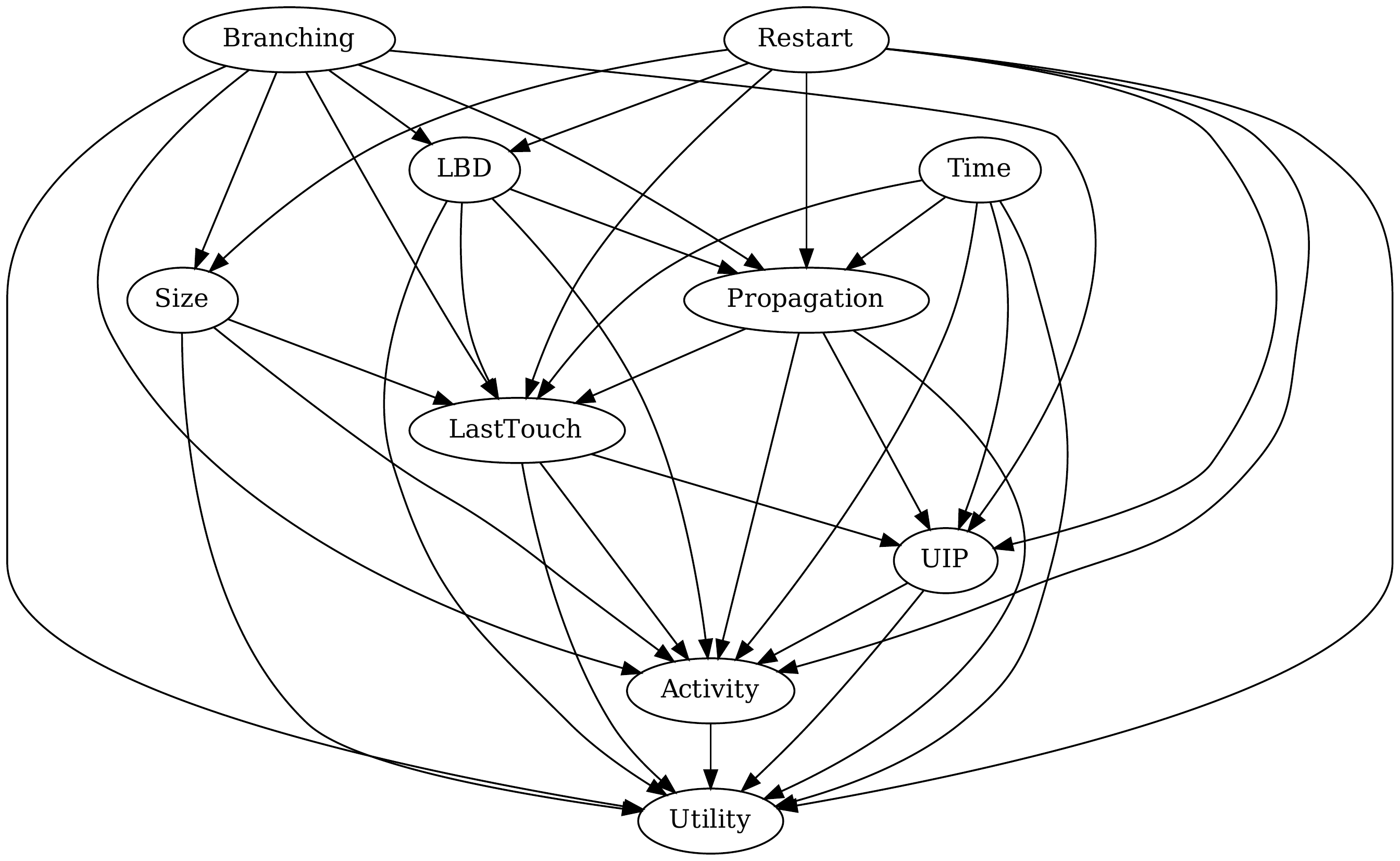}
	\caption{The causal graph {\name} derives using $10$-fold cross-validation and $4$ million points.}
	\label{fig: structure hc}
\end{figure}

\noindent\textbf{Summary.} 
Our model fits the dataset, attaining a comparable loss to both linear models and decision trees 
while exhibiting a higher Pearson correlation with clause utility than any individual variable.
With the model in place, we answer the questions in Table~\ref{tab:query-formulation}.
For LBD, a small value leads to a larger clause utility, 
but large value results in a smaller utility that drops rapidly over time.
For size, a large clause has greater utility, and the statement also holds when the LBD is fixed.
Overall, the LBD has a greater impact than the size on clause utility.
Besides LBD, size, and activity, the number of propagations has the greatest impact.
Last but not least, Maple has a greater utility between branching heuristics, and Luby achieves the highest utility among restart heuristics.

\subsection{Fitness on Data}
We evaluate the fitness of {\name} on the dataset by showing a comparable prediction error and a stronger correlation with clause utility than any individual variable.
We present the competence of {\name} with a linear model and decision tree by evaluating their performance based on the mean squared error (MSE) metric (the lower, the better).
We further demonstrate that the Pearson correlation of {\name} with clause utility is higher than that of any individual variable.

{\name} delivers a comparable prediction error with a linear model and decision tree.
We evaluate the mean squared error (MSE) of predicting clause utility for the three models.
We adopt 10-fold cross-validation where the dataset is split into ten subsets, and every time the model is trained on nine subsets and evaluated on the remaining subset.
The overall MSE is averaged over the 10 evaluations.
Table~\ref{tab:prediction-error} presents the MSE on different data sizes.
When 0.5 million data points are available, {\name} obtains an MSE of 10050.41, which is smaller than both the linear model and decision tree.
The result shows that {\name} still fits the data when we don't have much data available.
On the other hand, the MSE of the decision tree decreases significantly as the data size increases, and the decision tree finally becomes the best predictor.
Note that the prediction is not the focus of a causal model so it is expected to observe a smaller MSE from the decision tree when large data is available.

\begin{table}[h]
    \centering
    \begin{tabular}{c c c c}
        \toprule
        Data Size & Linear Model & Decision Tree & {\name} \\
        \midrule
        0.5 million & 10087.79 & 13400.43 & 10050.41 \\
        1 million & 9848.14 & 10731.84 & 9810.78 \\
        2 million & 9986.05 & 7555.20 & 9951.33 \\
        4 million & 9977.85 & 4335.32 & 9943.03 \\
        \bottomrule
    \end{tabular}
    \caption{Mean squared error in 10-fold cross-validation.}
    \label{tab:prediction-error}
\end{table}

Additionally, we evaluate the correlation with {\utility} for any individual variable and {\name}.
We compute the Pearson correlation coefficient between any individual variable and {\utility} except for {\branchheu} and {\restartheu}.
Pearson correlation coefficient is a measure of linear correlation between two variables and has a value between -1 and 1, with 1 (resp. -1) indicating a strong positive (resp. negative) linear relationship and 0 suggesting no linear relationship between the variables.
The Pearson correlation for {\name} measures the linear relationship between the {\utility} and the prediction generated by {\name}.
As shown in Table~\ref{tab:correlation}, {\name}'s prediction has a higher correlation with {\utility} than any individual variable.
Consequently, any individual variable is not a good indicator for {\utility} while {\name}, considering all variables and their causal relationship, delivers a better correlation with clause utility.

\begin{table}[h]
    \centering
    \begin{adjustbox}{max width=\textwidth}
    \begin{tabular}{c c c c c c c c c}
        \toprule
         & {\activity} & {\glue} & \lastouch & {\size} & {\prop} & {\uipuse} & {\timeinsolver} & {\name} \\
        \cmidrule(lr){2-8} \cmidrule(l){9-9}
        Correlation  & 0.2819 & -0.0285 & -0.0326 & -0.0283 & 0.2381 & 0.2557 & 0.0933 & 0.3425 \\
        \bottomrule
    \end{tabular}
    \end{adjustbox}
    \caption{Correlation with clause utility.}
    \label{tab:correlation}
\end{table}

\subsection{Query Answering}
We use {\name} to answer the questions in Table~\ref{tab:query-formulation} by calculating the causal effects for the queries.
Table~\ref{tab:query-answer} lists the calculations and answers,
where the causal estimates all pass the three refutation tests.
We use the original data to calculate ATE for a single variable due to its interpretability
while using normalized data in Q4 and Q5 to compare ATEs between variables to eliminate the unfairness of their different orders of magnitude.
We apply standard score normalization to the data.
As follows, we explain the answers in detail:
\begin{itemize}
\item (A1) LBD has a negative effect on clause utility, indicating that clauses with low LBD have greater utility. 
Hence, a low-LBD clause should be prioritized over a high-LBD clause in clause memory management.
\item (A2) Both ATE and ACATE values are negative, 
indicating that small clauses have greater utility, which also holds when LBD is fixed, though with a smaller effect. 
The answer to the fixed-LBD case contradicts the previous observation that a large clause is assumed to have greater utility when LBD is fixed.
\item (A3) {\timeinsolver} has a positive effect on {\utility} for a low-LBD clause and a negative effect for a high-LBD clause.
The results show that a high-LBD clause experiences a rapid drop in utility
while the utility of a low-LBD clause even increases over time.
Therefore, low-LBD clauses should be preserved in memory for a longer time than high-LBD clauses.
\item (A4) LBD has a larger absolute ATE value than size, indicating a greater impact on clause utility. The result suggests the usage of LBD in a clause memory management heuristic against the size from a causal perspective.
\item (A5) Appendix~\ref{sec:appendix-complete-effects} lists the ATE of all variables on clause utility over the normalized data. Among them, {\prop} has the greatest effect. The result suggests the introduction of {\prop} into the future design of clause memory management heuristics.
\item (A6) {\branchheu} has a positive effect on {\utility} when changing the strategy from VSIDS to Maple. Hence, Maple leads to greater utility and serves as a recommended branching heuristic for high-clause-utility scenarios.
\item (A7) Luby restart leads to a greater utility than Geometric and LBD-based, achieving the greatest utility. Hence, Luby is a recommended restart heuristic for high-clause-utility scenarios.
\end{itemize}

\begin{table}[h]
	\centering
	\begin{adjustbox}{max width=\textwidth}
	\begin{tabular}{c l l}
		\toprule
		& Query & Answer \\
		\midrule

            Q1 & $\ATE{\glue, \utility, 2, 1} = -0.2610 < 0 $ & Low-LBD clause has greater utility.\\ 
    
		Q2 & 
            $\begin{cases}
			     \ATE{\size, \utility, 2, 1} = -0.0314 < 0 \\
			     \ACATE{\size, \utility, \glue, 2, 1} = -0.0202 < 0 	
		\end{cases}$
            & \thead[l]{Small clause has greater utility, \\ which also holds when LBD is fixed.} \\
  
            Q3 
            & $\begin{cases}
		      \CATE{\timeinsolver, \utility, \glue \le 6, 10 000, 0} = 0.3772 > 0 \\
			\CATE{\timeinsolver, \utility, \glue > 6, 10 000, 0} = -0.0884 < 0
		\end{cases}$ 
            & \thead[l]{High-LBD clause experiences a rapid \\ drop in utility over time.} \\

            Q4
            & $|\ATE{\size, \utility, 2, 1}| = 2.6660 < |\ATE{\glue, \utility, 2, 1}| = 3.3754$
            & LBD has a greater impact than size. \\
  
            Q5
            & $\arg \max_{\ensuremath{\mathsf{Treatment}}\not = \utility} \{ |\ATE{\ensuremath{\mathsf{Treatment}}, \utility, 2, 1}| \} $
            & \thead[l]{Propagation has the greatest impact \\ on utility.} \\

            Q6
            & $\ATE{\branchheu, \utility, \text{Maple}, \text{VSIDS} } = 18.5745 > 0$ 
            & Maple leads to greater utility.\\

            Q7
            & $\begin{cases}
			\ATE{\restartheu, \utility, \text{Luby}, \text{Geometric} } = 6.7347 > 0\\
			\ATE{\restartheu, \utility, \text{Luby}, \text{LBD-based} } = 17.7385 > 0\\
                \ATE{\restartheu, \utility, \text{Geometric}, \text{LBD-based} } = 11.0037 > 0
		\end{cases}$ 
            & Luby leads to the greatest utility. \\
		\bottomrule	
	\end{tabular}
	\end{adjustbox}
	\caption{Query estimates and their interpretations.}
	\label{tab:query-answer}
\end{table}

\section{Conclusion} \label{sec:conclusion}
This work serves as the first work to utilize causality to uncover the inner workings of modern SAT solvers.
We proposed a framework to calculate the causal effects on clause utility from key factors such as LBD, size, activity, and the choice of heuristics.
The causal results verified the ``rules of thumb'' in solver implementation, for example, that a low-LBD clause is assumed to have greater utility than a high-LBD clause.
We also answered questions closely related to SAT solving, for instance, which factor, size or LBD, has a greater impact on clause utility?
The causal framework provides a systematical way to quantitatively investigate the relationship among components of a SAT solver,
which paves the way to further understanding how modern SAT solvers work.

This work opens several promising directions for future research. Our immediate focus will be enhancing our framework by introducing a new definition for clause utility. This revised definition will consider the frequency of a learned clause's usage by a SAT solver rather than its usage in DRAT proof generation. Additionally, We aim to include a definition for clause utility that encompasses both satisfiable and unsatisfiable instances, expanding beyond our current analysis of unsatisfiable cases.
Looking ahead, we intend to incorporate a broader range of features into our study and extend a further investigation into the obtained causal graph's structure and causal estimates' value to enrich our analysis.

\clearpage
\bibliography{main}

\clearpage
\appendix
\section{Hill-climbing Algorithm} \label{sec:hc-alg}

Algorithm~\ref{alg: hill climbing} illustrates the pseudo-code of the hill-climbing algorithm.
{\hillclimbing} takes three arguments: {\whitelist}, {\blacklist}, and {\data}, and returns a local-optimal causal graph fitting on {\data}.
{\whitelist} and {\blacklist} represent the prior knowledge about the causal graph.
The returned graph must contain all edges in {\whitelist} but contains none of the edges in {\blacklist}.
We use {\bestdiag} to store the best graph up to that time.
{\bestdiag} is initialized to an empty graph at Line~\ref{ln: hc init}.
Then, edges in {\whitelist} are added to {\bestdiag} at Lines~\ref{ln: hc add wl begin}-\ref{ln: hc add wl end}.
A loop is not allowed in a causal graph and we report {\fail} at Line~\ref{ln: hc cyclic fail} if the initialized graph is cyclic.
Each causal graph is associated with a score indicating the goodness of fit on {\data}.
We use Bayesian Information Criterion~\cite{S78} as the scoring function in our system.
{\bestscore} stores the score of {\bestdiag}.
Loop from Line~\ref{ln: hc loop begin} to~\ref{ln: hc loop end} finds a local optimal graph by iteratively moving to a nearby graph from {\bestdiag} in a greedy strategy.
{\currentdiag} stores the best graph from the previous iteration.
Line~\ref{ln: hc perturb} starts a loop over all possible perturbations of {\currentdiag}.
{\perturbdiag} stores the graph after a single perturbation.
{\perturbdiag} can result from adding an edge to {\currentdiag}, removing an edge from {\currentdiag}, or reverse an edge of {\currentdiag}.
If {\perturbdiag} passes the validation check at Line~\ref{ln: hc check},
we compute the score of {\perturbdiag} and compare the score to {\bestscore}.
If the score is better than {\bestscore},
we update {\bestscore} and {\bestdiag} accordingly.
Next, we continue the loop with the new \bestdiag.
The loop ends when none of the perturbed graphs is better than {\bestdiag},
which means the loop reaches a local optimal.
Finally, the local optimal graph, {\bestdiag}, is returned at Line~\ref{ln: hc return}.

\begin{algorithm}[htb]
	\caption{\hillclimbing$(\whitelist, \blacklist, \data)$}
	\label{alg: hill climbing}

	\begin{algorithmic}[1]
		\State $\bestdiag \gets \emptygraph();$  \label{ln: hc init}
		\For{$e \in \whitelist$} \label{ln: hc add wl begin}
			\State $\addedge(\bestdiag, e);$
		\EndFor \label{ln: hc add wl end}
		\If{$\bestdiag$ is cyclic}
			\State \Return \fail$;$ \label{ln: hc cyclic fail}
		\EndIf

		\State $\bestscore \gets \computescore(\bestdiag, \data);$
		\Repeat	\label{ln: hc loop begin}
			\State $\updated \gets \false;$
			\State $\currentdiag \gets \bestdiag;$
			\For{$\perturbdiag \gets \text{add, remove, or reverse an edge of } \currentdiag$} \label{ln: hc perturb}
				\If{$\perturbdiag$ is acyclic \pseudoand $\perturbdiag$ contains all edges in $\whitelist$ \pseudoand $\perturbdiag$ contains none of edges in $\blacklist$} \label{ln: hc check}
					\State $\score \gets \computescore(\perturbdiag, \data)$
					\If{$\score > \bestscore$}
						\State $\bestscore \gets \score;$
						\State $\bestdiag \gets \perturbdiag;$
						\State $\updated \gets \true;$
					\EndIf
				\EndIf
			\EndFor
		\Until{$\updated = \false$} \label{ln: hc loop end}

		\State \Return $\bestdiag;$ \label{ln: hc return}
	\end{algorithmic}
\end{algorithm}

\section{Identification and Estimation Algorithm} \label{sec:est-alg}

Algorithm~\ref{alg: estimate} presents the pseudo-code of the identification and estimation algorithm using backdoor and linear regression.
If there is no directed path from the treatment to the outcome in the causal graph,
we return zero at Line~\ref{alg: estimate zero}.
Otherwise, the algorithm first identifies a probability expression termed \emph{estimand} that can be evaluated over observational data to estimate the causal effect, and then uses linear regression to compute the estimand.
There are three methods to identify the expression: backdoor, frontdoor, and instrumental variable identification.
We only present the pseudo-code of backdoor identification at Lines~\ref{alg: backdoor begin}-\ref{alg: backdoor end}.
Interested readers may refer to~\cite{P09b} for the details of other methods.

For simplicity, we assume no unobserved variables in the graph.
A backdoor set is identified at Line~\ref{alg: backdoor set}, 
which is the smallest set \emph{d-separating} the treatment and outcome after removing the paths from the treatment to the outcome.
We say that two variables $X$ and $Y$ are d-separated by a disjoint subset $S$ if every path between $X$ and $Y$ is blocked by $S$~\cite{P09b}.
A backdoor set is a set of variables that affect the estimate of the causal effect from the treatment to the outcome.
Given a backdoor set $Z$, we can convert an ATE query into the following estimand:
\begin{align*}
    \ATE{X, Y, a, b} = \sum_z (\Ex{Y|X = a, Z = z} - \Ex{Y|X = b, Z = z})\Pr[Z=z].
\end{align*}
Following that, we run a linear regression to evaluate the estimand at Line~\ref{alg: linear regression}, 
and the coefficient of treatment scaled by $(a-b)$ is exactly the \emph{average treatment effect (ATE)} at Line~\ref{alg: lr treat coef}.

\begin{algorithm}[htb]
	\caption{$\idfyandest(\query, \causaldiag, \data)$}
	\label{alg: estimate}
	
	\begin{algorithmic}[1]
        \State $\treat, \outcome \gets \text{decode treatment and outcome variable from } \query;$
		\If{No directed path from {\treat} to {\outcome} in {\causaldiag}} 
			\State \Return $0;$ \label{alg: estimate zero} \Comment{Causal Effect is zero.} 
		\EndIf
		
		\State $\eligibleset \gets \emptyset;$ \label{alg: backdoor begin}
		\For{$\bdoorvar \gets$ any node in {\causaldiag}}
			\If{{\bdoorvar} is not {\treat} \pseudoand {\bdoorvar} is not {\outcome}}
				\If{{\bdoorvar} is not a descendant of {\treat}}
					\If{{\bdoorvar} is not d-separated from {\treat} \pseudoor {\bdoorvar} is not d-separated from {\outcome}}
						\State $\addtolist(\eligibleset, \bdoorvar);$ 
					\EndIf
				\EndIf
			\EndIf
		\EndFor
		
		\State $\bdoorsize \gets 1;$
		\State $\bdoorgraph \gets$ remove outgoing edges of {\treat} in {\causaldiag}$;$
		\Repeat
			\For{$\candset \gets$ any {\bdoorsize} variable(s) from {\eligibleset}}
				\If{{\treat} and {\outcome} are d-separated by {\candset} in {\bdoorgraph}}
					\State $\bdoorset \gets \candset;$ \label{alg: backdoor set}
					\State {\pseudobreak}$;$
				\EndIf
			\EndFor
			\State $\bdoorsize \gets \bdoorsize + 1;$
		\Until{$\bdoorsize > \sizeof(\eligibleset)$} \label{alg: backdoor end}
		
		\State $\coefs \gets \linreg(\outcome, \treat+\bdoorset);$ \label{alg: linear regression}
		\State $\estimates \gets \getcoef(\coefs, \treat) \times (a-b);$ \label{alg: lr treat coef}
		\Comment{get {\treat}'s coefficient.}
		\State \Return $\estimates;$
	\end{algorithmic}
\end{algorithm}

\section{Proof of Lemma~\ref{lm:cate}} \label{sec:appendix-proof}

We present a complete proof for Lemma~\ref{lm:cate}, considering the general case when $W$ and $Z$ are continuous variables:
\begin{proof}
        Following the definition of the backdoor set, we have 
        \begin{align*}
            X \coprod (Y_a, Y_b) \mid Z \cup W,
        \end{align*}
        where $Y_i$ denotes $Y \mid do(X = i)$. Then, we can obtain
	\begin{align}
		\ACATE{X, Y, W, a, b} &= \int_w \Exp{Y_a - Y_b \mid W=w} p(w) dw \nonumber \\
		&= \int_w \int_z \Exp{Y_a - Y_b \mid Z=z, W=w} p(z) dz \cdot p(w) dw \nonumber \\
		&= \int_w \int_z \Exp{Y_a \mid X=a, Z=z, W=w} p(z) dz \cdot p(w) dw \nonumber \\
		&- \int_w \int_z \Exp{Y_b \mid X=b, Z=z, W=w} p(z) dz \cdot p(w) dw \label{eq:addx}\\
		&= \int_w \int_z \Exp{Y \mid X=a, Z=z, W=w} p(z) dz \cdot p(w) dw \nonumber \\
		&- \int_w \int_z \Exp{Y \mid X=b, Z=z, W=w} p(z) dz \cdot p(w) dw \nonumber \\
		&= \int_w \int_z \left ( \beta_0 + \beta_1a + \beta^T_2 z + \beta^T_3 w \right ) p(z) dz \cdot p(w) dw \nonumber \\
		&- \int_w \int_z \left ( \beta_0 + \beta_1b + \beta^T_2 z + \beta^T_3 w \right ) p(z) dz \cdot p(w) dw \nonumber \\
		&= (a-b)\beta_1 \nonumber 
	\end{align}
	where~\Cref{eq:addx} follows from $X \coprod (Y_a, Y_b) \mid Z \cup W$, and $p(z)$ (resp. $p(w)$) represents the probability density function of variable $Z$ as $Z=z$ (resp. variable $W$ as $W=w$).
\end{proof}

Following that, we present the proof for ATE, while CATE can be considered an ATE over the observation data filtered by the condition.
\begin{lemma}
	\label{lm:ate}
	Suppose $Z$ is a backdoor set and $\Exp{Y \mid X = x, Z = z}$ is linear, \\
        i.e., $\Exp{Y \mid X = x, Z = z} = \beta_0 + \beta_1 x + \beta^T_2 z$. 
	We obtain 
	\begin{align*}
		\ATE{X, Y, a, b} = \Ex{Y \mid do(X = a)} - \Ex{Y \mid do(X = b)} = \beta_1
	\end{align*}
\end{lemma}
\begin{proof}
        Following the definition of the backdoor set, we have 
        \begin{align*}
            X \coprod (Y_a, Y_b) \mid Z,
        \end{align*}
        where $Y_i$ denotes $Y \mid do(X = i)$. Then, we can obtain
	\begin{align}
		\ATE{X, Y, a, b} &= \Exp{Y_a - Y_b} \nonumber \\
		&= \int_z \Exp{Y_a - Y_b \mid Z=z} p(z) dz \nonumber \\
		&= \int_z \Exp{Y_a \mid X=a, Z=z} p(z) dz \nonumber \\
		&- \int_z \Exp{Y_b \mid X=b, Z=z} p(z) dz \label{eq:addx-ate}\\
		&= \int_z \Exp{Y \mid X=a, Z=z} p(z) dz \nonumber \\
		&- \int_z \Exp{Y \mid X=b, Z=z} p(z) dz \nonumber \\
		&= \int_z \left ( \beta_0 + \beta_1a + \beta^T_2 z \right ) p(z) dz \nonumber \\
		&- \int_z \left ( \beta_0 + \beta_1b + \beta^T_2 z \right ) p(z) dz \nonumber \\
		&= (a-b)\beta_1 \nonumber 
	\end{align}
	where~\Cref{eq:addx-ate} follows from $X \coprod (Y_a, Y_b) \mid Z$, and $p(z)$ represents the probability density function of variable $Z$ as $Z=z$.
\end{proof}

\section{Refutation Test} \label{sec:appendix-refute}

We conduct diverse refutation tests to enhance confidence in our statistical estimations.
Our system incorporates three types of tests: adding random common cause, placebo treatment, and data subsets validation.
Each test outputs a p-value, and we have enough evidence to refute the estimation should the p-value be less than 0.05.

The first approach assesses whether the estimate alters substantially upon adding an independent random variable to the data as a common cause of the treatment and outcome. 
If the estimation method is effective, it should yield a similar estimate as previously. 
Should the p-value be less than 0.05, we have enough evidence to say the new and prior estimates are different and refute the estimation.
The placebo treatment method evaluates whether the estimate approaches zero when the treatment is substituted with an independent random variable. 
If the p-value is less than 0.05, the new estimate statistically deviates from zero, and the estimation is refuted.
The final test involves sampling a random subset from the data. 
A less-than-0.05 p-value indicates the new estimate on the subset significantly diverges from the prior estimate, and hence the estimation is rejected.

If the estimate successfully passes all three refutation tests, we consider it reliable and return it as the causal effect. 
Conversely, we report a failure if the estimate is refuted by any refutation test.

\section{Causal Effects on Clause Utility} \label{sec:appendix-complete-effects}

Table~\ref{tab:all-on-utility} lists the ATEs of all variables on {\utility}.
We apply standard score normalization to the dataset in order to compare effects among different variables without bias from different orders of magnitude.
\begin{table}[h]
	\centering
	\begin{tabular}{l c c}
		\toprule
		Treatment & Effect & Normalized Effect \\
		\midrule
		{\activity} & 4.1246 & 22.0866 \\
		  {\branchheu}: VSIDS $\rightarrow$ Maple & 18.5745 & 18.5745 \\
		{\glue} & -0.2610 & -3.3754 \\
		{\lastouch} & $-1.8921 \times 10^{-5}$ & -2.2492 \\
		{\restartheu}: LBD-based $\rightarrow$ Geometric & 11.0037 & 11.0037 \\
		{\restartheu}: Geometric $\rightarrow$ Luby & 6.7347 & 6.7347 \\
		{\restartheu}: LBD-based $\rightarrow$ Luby & 17.7385 & 17.7385\\
		{\size} & -0.0314 & -2.6660 \\
		{\prop} & 0.0065 & 25.0221 \\
		{\uipuse} & 0.0187 & 18.5992 \\
		{\timeinsolver} & $3.9960 \times 10^{-5}$ & 9.8995 \\
		\bottomrule
	\end{tabular}
	\caption{$\ATE{\ensuremath{\mathsf{Treatment}}, \utility, 2, 1}$ over original and normalized data. We apply standard score normalization to the dataset.}
	\label{tab:all-on-utility}
\end{table}

\end{document}